\documentclass[letterpaper]{article} 
\usepackage{aaai2027}  
\usepackage[hyphens]{url}  
\usepackage{graphicx} 
\usepackage{natbib}  
\usepackage{caption} 
\usepackage{algorithm}
\usepackage{algorithmic}

\usepackage{newfloat}
\usepackage{listings}
\DeclareCaptionStyle{ruled}{labelfont=normalfont,labelsep=colon,strut=off} 
\floatstyle{ruled}
\newfloat{listing}{tb}{lst}{}
\floatname{listing}{Listing}

\usepackage{booktabs}

\usepackage{multirow} 
\usepackage{booktabs} 
\usepackage{multirow}
\usepackage{amsmath}
\usepackage{graphicx}   
\usepackage{multirow}   
\usepackage{array}      
\usepackage{booktabs}
\usepackage{multirow}
\usepackage{makecell}
\usepackage{graphicx}
\usepackage{tabularx}
\usepackage{multirow}    
\usepackage{booktabs}    
\usepackage{graphicx} 
\usepackage{booktabs}    
\usepackage{multirow}    
\usepackage{graphicx}    
\usepackage{amsmath}
\usepackage{algorithm}
\usepackage{algorithmic}

\usepackage{newfloat}
\usepackage{listings}
\usepackage{amsmath} 
\usepackage{amssymb}
\usepackage{multirow}
\usepackage{graphicx} 
\usepackage{natbib}  
\usepackage{caption} 
\usepackage{booktabs}
\usepackage{multirow} 
\usepackage{booktabs} 
\usepackage{multirow}
\usepackage{amsmath}
\usepackage{graphicx}   
\usepackage{multirow}   
\usepackage{array}      
\usepackage{booktabs}
\usepackage{multirow}
\usepackage{makecell}
\usepackage{graphicx}
\usepackage{tabularx}
\usepackage{multirow}    
\usepackage{booktabs}    
\usepackage{graphicx} 
\usepackage{booktabs}    
\usepackage{multirow}    
\usepackage{graphicx}    
\usepackage{amsmath}
\usepackage{amsthm}
\usepackage{algorithm}
\usepackage{algorithmic}
\usepackage{amssymb}
\usepackage{multirow}
\newtheorem{theorem}{Theorem}[section] 

\newtheorem{corollary}[theorem]{Corollary}

\usepackage{newfloat}
\usepackage{listings}
\DeclareCaptionStyle{ruled}{labelfont=normalfont,labelsep=colon,strut=off} 
\floatstyle{ruled}
\newfloat{listing}{tb}{lst}{}
\floatname{listing}{Listing}

\usepackage{booktabs}

\title{FineSID: Scalable and Efficient Semantic Identifier Learning \\for Generative Recommendation}
\author{
    Song-Li Wu\textsuperscript{\rm 1}\equalcontrib,
    Weinan Gan\textsuperscript{\rm 2}\equalcontrib,
    Zhaocheng Du\corresponding\textsuperscript{\rm 2},
    Xianquan Wang\textsuperscript{\rm 3},
    Jingyi Wang\textsuperscript{\rm 1}
}
\affiliations{
    \textsuperscript{\rm 1}Tsinghua University,
    \textsuperscript{\rm 2}Huawei Noah’s Ark Lab,
    \textsuperscript{\rm 3}University of Science and Technology of China
    
}

\begin{document}

\maketitle

\begin{abstract}
A critical prerequisite of generative recommendation is designing semantic identifiers (SIDs) that are both scalable to large item sets and efficiently learnable. Existing SID learning methods fundamentally rely on Top‑1 hard assignment during vector quantization. While heuristic strategies—such as clustering-based initialization or forced post-hoc collision resolution—can artificially inflate codebook coverage, they often disrupt end-to-end semantic alignment and fail to address the underlying optimization bottleneck: sparse gradient propagation. In standard Top-1 assignment, gradients concentrate on a narrow subset of frequently selected codewords, leaving the majority inherently under-trained and causing severe SID collisions. To overcome this limitation natively without relying on complex initialization priors, we propose FineSID, a unified quantization framework that moves beyond Top‑1 assignment by enabling fine‑grained gradient propagation across the entire codebook. Instead of updating only a single selected codeword, FineSID distributes learning signals to all codewords in a soft, differentiable manner. This design promotes globally balanced codebook optimization while strictly preserving semantic consistency, effectively alleviating SID collisions and stabilizing training in large, high‑dimensional codebooks. Extensive experiments on multiple public benchmarks demonstrate that FineSID is robust to initialization configurations and consistently improves both codebook utilization and recommendation accuracy. Our work provides a principled, initialization-agnostic solution for semantic identifier learning, advancing the practicality of generative recommendation. Codes are available.
\end{abstract}

\section{Introduction}

Framing recommendation as an autoregressive sequence generation problem has emerged as a promising paradigm for recommendation, enabling unified modeling of users, items, and their interactions~\citep{rajput2023recommender,wu2026activemem}.
Central to this paradigm is the Semantic Identifier (SID), which bridges continuous item representations and discrete token generation by learning a quantization codebook, typically instantiated via vector-quantized autoencoders such as RQ-VAE~\citep{wang2024learnable}.
By transforming recommendation into a token generation process, SIDs provide a scalable foundation for expressive generative recommendation models.

However, existing SID learning methods predominantly rely on Top-1 hard assignment during discretization, which induces a fundamental optimization bottleneck: gradient updates concentrate on a narrow subset of frequently selected codewords, leaving the majority under-trained~\citep{rajput2023recommender, wang2024learnable}.
While in practice one might employ heuristic strategies to mitigate these symptoms—such as utilizing K-means clustering for robust codebook initialization or resolving residual collisions by artificially assigning items to the next closest available code—these workarounds fail to address the root optimization bottleneck.
Specifically, static initialization strategies provide a favorable starting point but cannot prevent codebook collapse during dynamic, end-to-end training, as Top-1 assignment continually starves low-frequency codes of gradient flow.
Furthermore, arbitrarily reassigning items to sub-optimal, semantically distant codes merely to avoid collisions fundamentally disrupts the semantic fidelity of the identifiers, forcing the generative model to learn distorted item representations.
Consequently, the reliance on Top-1 assignment inevitably leads to a trade-off between severe SID collisions and semantic drift, ultimately degrading both item reconstruction fidelity and downstream recommendation performance—even when the codebook size is substantially increased~\citep{guo2026promise,wu2026capedit}.

To alleviate SID collisions, prior work has explored two primary directions.
One line of research introduces auxiliary regularization objectives to encourage more balanced codebook utilization~\citep{wang2024learnable, yao2025saviorrec}.
Another line enhances representational expressiveness by augmenting SIDs with additional tokens, thereby expanding the identifier space~\citep{rajput2023recommender, wang2025act,wu2026fairdiff}.
However, an effective Semantic Identifier (SID) for generative recommendation must simultaneously satisfy two essential properties: scalability and efficiency~\citep{kong2025minionerec, liu2025understanding,wang2026grpo}.
Scalability requires the SID space to accommodate large item catalogs with low collision rates, while efficiency demands computationally lightweight autoregressive SID generation to ensure low inference overhead.
From this perspective, existing solutions exhibit fundamental limitations.
Regularization-based methods promote more uniform code usage but often suppress intrinsic semantic differentiation among codewords, effectively shrinking the usable representational capacity of the codebook—particularly for long-tailed or rare items—thereby limiting scalability to large item sets~\citep{Li2025SurveyGenerativeRecommendation,wu2026id}.
In contrast, auxiliary-token approaches alleviate collisions by lengthening SID sequences, which directly increases autoregressive decoding cost and undermines inference efficiency~\citep{wang2025act,wu2026reaction}.
These limitations jointly motivate a principled SID learning paradigm that performs global optimization over the code space, preserving semantic structure while simultaneously achieving both scalability and efficiency.

To resolve these challenges, we propose FineSID, a unified quantization framework for principled semantic identifier optimization.
FineSID introduces a full-codebook optimization mechanism that overcomes the Top-1 hard assignment bottleneck, enabling learning signals to propagate across all codewords while strictly preserving discrete identifiers for autoregressive generation.
FineSID consists of two tightly coupled components: Global–Local Quantization (GLQ) and a Semantic Consistency Module.
GLQ maintains a soft SID for global optimization alongside a hard SID for discrete generation, thereby balancing expressive codebook capacity with high-fidelity autoregressive decoding.
The Semantic Consistency Module enforces alignment between discrete SIDs and continuous item representations, preventing semantic drift—particularly for low-frequency codewords.
Together, these designs enable FineSID to distribute semantic learning signals across both frequent and long-tailed items, ensuring stable optimization and efficient inference even with large, high-dimensional codebooks.
As a result, FineSID simultaneously achieves scalability to large item sets and efficiency in autoregressive SID generation.

We summarize our contributions as follows:
\begin{itemize}
\item We identify \emph{Top-1 hard assignment} as a fundamental optimization bottleneck in SID learning, revealing that while heuristic initializations and greedy re-assignments can mask symptoms, they inherently trade codebook collapse for semantic distortion, limiting true scalability in generative recommendation.
\item We propose \textbf{FineSID}, a unified quantization framework that enables \emph{global, fine-grained gradient propagation} across the entire codebook while strictly preserving discrete compatibility for autoregressive generation, thereby jointly addressing scalability and efficiency.
\item We design two tightly coupled components—\emph{Global--Local Quantization (GLQ)} and a \emph{Semantic Consistency Module}—that support stable optimization and semantic alignment, particularly for long-tailed items.
\item Extensive experiments on multiple public benchmarks demonstrate that FineSID consistently outperforms existing SID-based methods in both codebook utilization and recommendation accuracy, with particularly pronounced gains under long-tailed item distributions.
\end{itemize}

\section{Related Work}

\subsection{Generative Recommendation Systems}
Generative Recommendation (GR) formulates recommendation as an autoregressive sequence generation problem by representing items with discrete Semantic IDs (SIDs)~\citep{rajput2023recommender,li2025matching,sun2023learning}. A typical GR pipeline involves embedding extraction, SID quantization, and generative model training~\citep{petrov2025efficient,lin2025unified}. Among these, SID quantization is the critical bridge mapping continuous representations to discrete generation tokens, directly determining SID uniqueness, semantic fidelity, and overall recommendation performance~\citep{deldjoo20241st,li2024generative,wu2026codeskill}.

\subsection{SID Collision in Generative Recommendation}
SID collision occurs when distinct items map to identical SID sequences, causing ambiguous item grounding. Prior work mitigates this via two main strategies: (1) training-time regularization to encourage balanced codeword utilization (e.g., entropy-regularized assignment~\citep{yao2025saviorrec} or constrained codebooks~\citep{wang2024learnable}), and (2) augmenting SIDs with auxiliary tokens to expand the representational space~\citep{rajput2023recommender,wang2025act,chen2025onesearch}.
Additionally, while soft quantization methods (e.g., Straight-Through Estimator or Gumbel-Softmax)~\citep{bengio2013estimating, jang2016categorical} enable global gradient flow, they inherently smooth discrete representations and weaken strict item boundaries. This makes them unsuitable for SIDs, which must function as discrete, atomic symbols for autoregressive generation.
Ultimately, GR imposes dual constraints on SID learning: SIDs must be strictly discrete and uniquely identifiable for autoregressive decoding, while being globally optimizable and robust to long-tailed items. Existing methods typically satisfy only a subset of these requirements—often trading semantic fidelity for balanced usage, or inference efficiency for uniqueness. Consequently, SID collision remains a persistent challenge, motivating the need for more principled, globally optimized SID quantization frameworks.

\begin{figure*}[tb]
\centering
\includegraphics[width=0.9\textwidth]{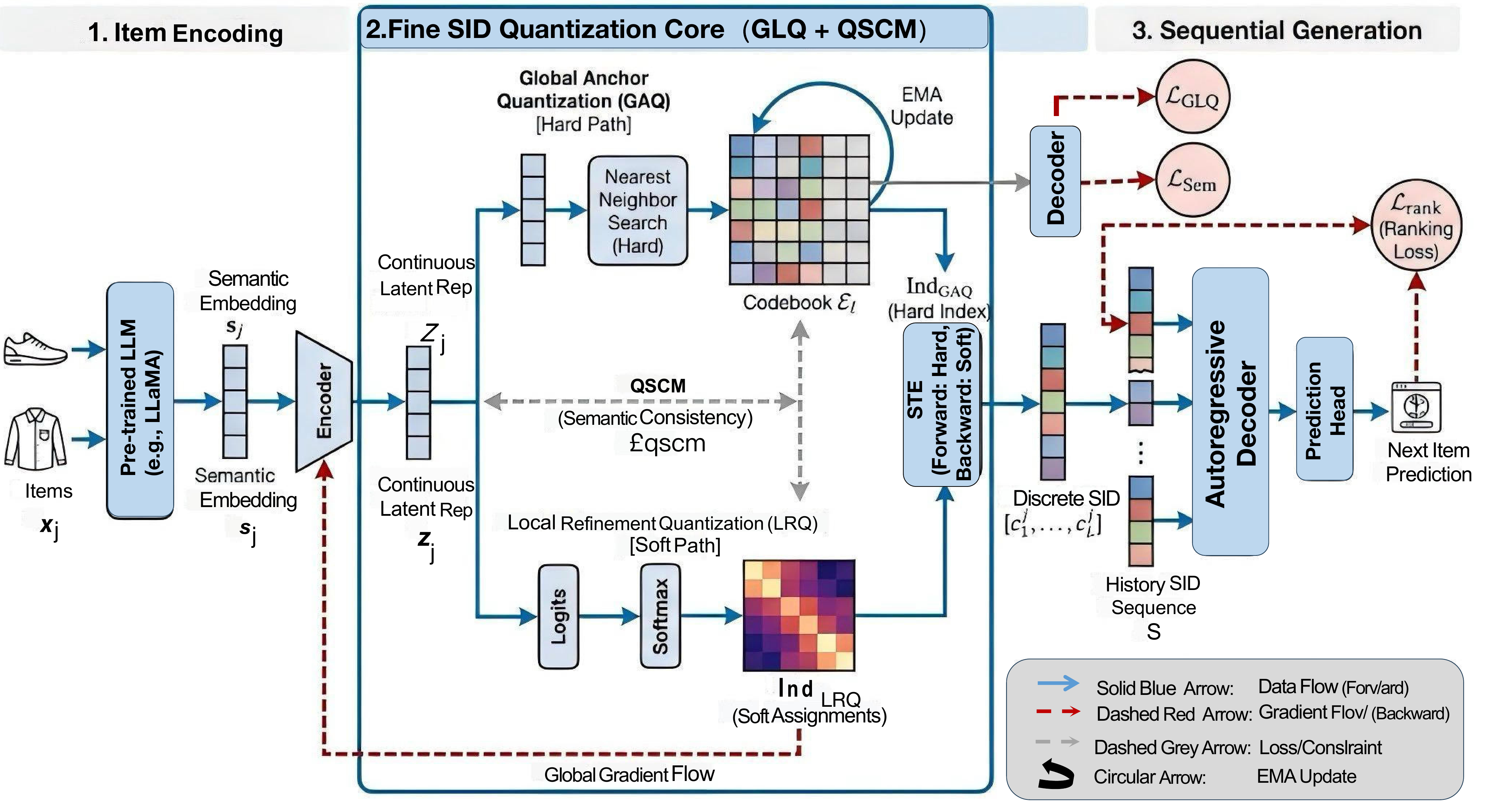}
\caption{Overview of FineSID.}
\label{Overview}
\end{figure*}

\section{Methodology}
To address SID collision in generative recommendation, we propose FineSID, a unified quantization framework that performs principled optimization over the semantic space. FineSID consists of two complementary modules: 
Global--Local Quantization (GLQ) and the Quantization Semantic Consistency Module (QSCM). 
Specifically, \textbf{GLQ} stabilizes and balances SID usage at the \emph{structural level}. It achieves this by combining local, dense gradient updates with global codebook usage tracking to prevent codebook collapse. Meanwhile, \textbf{QSCM} acts at the \emph{representation level} to preserve semantic fidelity, ensuring that the continuous embeddings remain tightly aligned with their discrete SIDs even as the codebook updates. Together, these modules ensure scalable, efficient, and semantically faithful SID learning for autoregressive recommendation.

\subsection{Problem Formulation and Background}
Let $\mathcal{I}$ denote the universal item set. The interaction history is represented as a sequence $S = [x_1, x_2, \dots, x_t]$ with $x_t \in \mathcal{I}$. The objective is to predict the next item $x_{t+1}$, which we formulate as a token-based sequence generation task.

\noindent\textbf{Item Tokenization.} 
Instead of representing items with opaque ID embeddings, we tokenize each item into a sequence of Semantic Identifiers (SIDs). This design provides two key benefits: it enables generalization to unseen items and ensures compatibility with autoregressive generative models. For each item $x_j$, we first extract a dense semantic embedding $\mathbf{s_j} \in \mathbb{R}^d$ from its content features (e.g., title or description) using a pre-trained language model such as LLaMA-7B~\citep{touvron2023llama}. This embedding is then encoded into a latent representation $\mathbf{z_j} = \mathrm{Encoder}(\mathbf{s_j})$. 

To discretize $\mathbf{z_j}$, we employ a Residual Quantized VAE (RQ-VAE) tokenizer~\citep{lee2022autoregressive} spanning $L$ quantization levels. We perform residual quantization iteratively. Starting from the initial residual $\mathbf{r}^{(0)}_j = \mathbf{z_j}$, each level $l$ selects the nearest SID from the codebook $\mathcal{E}_l = [\mathbf{e}^{(l)}_1, \dots, \mathbf{e}^{(l)}_K]$:
\begin{equation}
c^j_l = \arg\min_{i \in \{1, \dots, K\}} \|\mathbf{r}^{(l-1)}_j - \mathbf{e}^{(l)}_i\|^2, \quad
\mathbf{r}^{(l)}_j = \mathbf{r}^{(l-1)}_j - \mathbf{e}^{(l)}_{c^j_l}.
\end{equation}
The item is thus tokenized into an SID sequence: $[c^j_1, c^j_2, \dots, c^j_L]$. The quantized embedding $\hat{\mathbf{z_j}}=\sum_{l=1}^L  \mathbf{e}^{(l)}_{c^j_l}$ is then decoded to reconstruct the semantic embedding $\hat{\mathbf{s_j}}$. A reconstruction loss ensures semantic faithfulness:
\begin{align}
\mathcal{L}_{\mathrm{Sem}} =& \|\mathbf{s_j} - \hat{\mathbf{s_j}}\|^2 + \sum_{l=1}^{L} \left( \|\mathrm{sg}[\mathbf{r}^{(l-1)}_j] - \mathbf{e}^{(l)}_{c^j_l}\|^2 \right. \notag \\
&+ \left.  \|\mathbf{r}^{(l-1)}_j - \mathrm{sg}[\mathbf{e}^{(l)}_{c^j_l}]\|^2 \right),
\end{align}
where $\mathrm{sg}[\cdot]$ is the stop-gradient operator. The full user interaction history $S$ thus becomes a flattened token sequence $ [c^1_1, \dots, c^1_L, \dots, c^t_1, \dots, c^t_L]$.

\noindent\textbf{Autoregressive Generation.}  
The model predicts the next-item token sequence $Y = [c^{t+1}_1, \dots, c^{t+1}_L]$ autoregressively. To strengthen ranking performance, we use a temperature-scaled softmax cross-entropy loss~\citep{wang2024learnable}:
\begin{align}
&\mathcal{L}_{\mathrm{rank}} = -\frac{1}{|Y|} \sum_{l=1}^{L} \log P_\theta(c^{t+1}_l \mid c^{t+1}_{<l}, X),\\
&P_\theta(c^{t+1}_l \mid c^{t+1}_{<l}, X) = \frac{\exp(p(c^{t+1}_l))}{\sum_{v \in V} \exp(p(v))},
\end{align}
where $V$ is the complete token vocabulary.

\subsection{Global-Local Quantization (GLQ)}
Traditional Top-1 hard assignment restricts gradient updates to a single winning codeword per item. This bottleneck results in sparse gradients, causing many SIDs to become inactive early in training (codebook collapse). To address this, we introduce Global–Local Quantization (GLQ), consisting of two complementary mechanisms: LRQ for dense gradient flow and GAQ for global frequency balancing.

\paragraph{Local Refinement Quantization (LRQ)}
The role of LRQ is to provide dense, token-level gradient propagation. By distributing learning signals across the entire codebook—including rarely used tokens—LRQ effectively circumvents the gradient sparsity problem. Specifically, LRQ computes a soft assignment probability distribution between the input residuals and all SIDs:
\begin{align}
\text{Ind}_{\mathrm{LRQ}(t)}^{(l)} =  \text{softmax}(\text{logits}_{(t)}^{(l)}),
\end{align}
where $\text{logits}_{(t)}^{(l)} \in \mathbb{R}^K$ represents the similarities (e.g., negative scaled distance) between the item's residual and all codewords in $\mathcal{E}_l$ at training step $t$, and $\text{Ind}_{\mathrm{LRQ}(t)}^{(l)}$ is the resulting soft assignment vector. 

Using the straight-through estimator (STE), LRQ maintains differentiable quantization during training while strictly outputting discrete hard assignments for inference:
\begin{align}
\text{Ind}^{(l)}_{(t)} = \text{Ind}_{\text{GAQ}(t)}^{(l)} - \mathrm{sg}[\text{Ind}_{\mathrm{LRQ}(t)}^{(l)} ] + \text{Ind}_{\mathrm{LRQ}(t)}^{(l)}.
\end{align}
Here, $\text{Ind}_{\text{GAQ}(t)}^{(l)}$ denotes the one-hot encoded vector of the actual Top-1 hard assignment. This formulation allows the forward pass to utilize the hard index (from GAQ), while the backward pass routes gradients through the soft probabilities (from LRQ), preventing "dead" tokens.

\paragraph{Global Anchor Quantization (GAQ)}
While LRQ ensures all codes receive gradients, low-frequency SIDs may still be overwhelmed by gradients from high-frequency ones. GAQ acts as a global regulator, dynamically adjusting codeword update rates based on their historical utilization to ensure long-term balance. It utilizes an exponential moving average (EMA) of usage frequencies:
\begin{align}
N_{k(t)}^{(l)} = \gamma N_{k(t-1)}^{(l)} + (1-\gamma)\frac{n_{k(t)}^{(l)}}{B},
\end{align}
\begin{align}
\mathbf{e}_{k(t)}^{(l)} = (1-\alpha_{k(t)}^{(l)}) \mathbf{e}_{k(t-1)}^{(l)} + \alpha_{k(t)}^{(l)} \hat{\mathbf{z}}_{k(t)}^{(l)}.
\end{align}
In these equations, $N_{k(t)}^{(l)}$ represents the smoothed usage frequency of the $k$-th codeword at level $l$ and step $t$, $\gamma \in (0,1)$ is the EMA decay factor, and $n_{k(t)}^{(l)}$ is the actual number of times the $k$-th codeword was selected in the current batch of size $B$. For the codeword update, $\mathbf{e}_{k(t)}^{(l)}$ is the $k$-th codeword embedding, $\alpha_{k(t)}^{(l)}$ acts as an adaptive update rate (inversely scaled by the usage frequency $N_{k(t)}^{(l)}$), and $\hat{\mathbf{z}}_{k(t)}^{(l)}$ is the aggregated mean of the latent residuals assigned to the $k$-th codeword in the current batch. This mechanism continuously steers underutilized codes toward high-density semantic regions.

\subsection{Quantization Semantic Consistency Module (QSCM)}
Because GLQ distributes updates across the codebook, there is a risk that codes may drift from their original semantic meanings. The role of QSCM is to enforce bidirectional alignment between the continuous embeddings and their discrete semantic anchors, ensuring semantic faithfulness throughout training. 

QSCM bridges this gap by imposing a \emph{quantization semantic consistency loss}:
\begin{align}
\hat{ \mathbf{z}}_{j(t)}&=\sum_{l=1}^L \text{Ind}_{\text{GAQ}{(t)}}^{(l)} \mathcal{E}_l,\\
\mathcal{L}_{\mathrm{QSCM}}
&= \|\hat{ \mathbf{z}}_{j(t)} - \mathrm{sg}[\mathbf{z}_j]\|^2
+ \|\mathrm{sg}[\hat{ \mathbf{z}}_{j(t)}] - \mathbf{z}_j\|^2.
\end{align}
Here, $\hat{ \mathbf{z}}_{j(t)}$ denotes the fully quantized latent representation reconstructed using the one-hot hard assignments $\text{Ind}_{\text{GAQ}{(t)}}^{(l)}$. The first term of the loss pulls the quantized semantic anchors toward the original continuous embedding, refining the codebook to capture input semantics accurately. The second term anchors the continuous embedding to its quantized counterpart, preventing semantic drift during gradient updates. By explicitly regularizing this pathway, QSCM provides the representational stability needed to support GLQ's structural balancing.

\subsection{Optimization Objectives}
In summary, FineSID integrates the structural-level stability of GLQ with the semantic-level alignment of QSCM. GLQ dynamically balances SID utilization to prevent representation collapse, while QSCM enforces semantic faithfulness. The overall framework is trained by jointly optimizing three complementary objectives:
\begin{align}
\mathcal{L}_{\textbf{FineSID}} &= \mathcal{L}_{\mathrm{rank}} + \mathcal{L}_{\mathrm{GLQ}} +  \mathcal{L}_{\mathrm{QSCM}}.
\end{align}
The autoregressive ranking loss $\mathcal{L}_{\mathrm{rank}}$ aligns SIDs with sequential user–item interaction patterns, the item representation reconstruction loss $\mathcal{L}_{\mathrm{GLQ}}$ ensures that SIDs retain the original item content, and $\mathcal{L}_{\mathrm{QSCM}}$ enhances discriminability by strictly aligning the continuous and quantized spaces.
\section{Experiment}
\subsection{Experiment Setup}
\paragraph{Dataset.} To evaluate our method, we experiment on three datasets~\citep{hou2024bridging}, including “Musical Instruments,” “Video Games,” and “Industrial Scientific.” These datasets cover user reviews from May 1996 through September 2023. In line with previous studies~\citep{zheng2024adapting,zhou2020s3}, we filter out users and items with fewer than five interactions using the 5-core criterion. User behavior sequences are then ordered chronologically, and the maximum sequence length for items is capped at 50. Table ~\ref{tab:dataset_stats} provides the statistics of the processed datasets.
\begin{table}[ht]
\centering
\caption{Statistics of the Datasets.}
\begin{tabular}{lcccc}
\toprule
Dataset & \#Users & \#Items & \#Interactions & Sparsity \\
\midrule
Instrument & 57,439 & 24,587 & 511,836 & 99.964\% \\
Scientific & 50,985 & 25,848 & 412,947 & 99.969\% \\
Game & 94,762 & 25,612 & 814,586 & 99.966\% \\
\bottomrule
\end{tabular}
\label{tab:dataset_stats}
\end{table}

\begin{table*}[ht]
\centering
\caption{The overall performance comparisons between the baselines and FineSID. The best and second-best results are highlighted in \textbf{bold} and underlined font, respectively. The improvement is statistically significant with $p < 10^{-2}$ ($\star$: $p < 10^{-2}$, $\star\star$: $p < 10^{-4}$).}
  \resizebox{\textwidth}{!}{%
\begin{tabular}{lcccccccccccc}
\toprule
\multirow{2}{*}{Method} & \multicolumn{4}{c}{Instrument} & \multicolumn{4}{c}{Scientific} & \multicolumn{4}{c}{Game} \\
\cmidrule(lr){2-5} \cmidrule(lr){6-9} \cmidrule(lr){10-13}
 & Recall@5 & Recall@10 & NDCG@5 & NDCG@10 & Recall@5 & Recall@10 & NDCG@5 & NDCG@10 & Recall@5 & Recall@10 & NDCG@5 & NDCG@10 \\
\midrule
Caser & 0.0242 & 0.0392 & 0.0154 & 0.0202 & 0.0172 & 0.0281 & 0.0107 & 0.0142 & 0.0346 & 0.0567 & 0.0221 & 0.0291 \\
GRU4Rec & 0.0345 & 0.0537 & 0.0220 & 0.0281 & 0.0221 & 0.0353 & 0.0144 & 0.0186 & 0.0522 & 0.0831 & 0.0337 & 0.0436 \\
HGN & 0.0319 & 0.0515 & 0.0202 & 0.0265 & 0.0220 & 0.0356 & 0.0138 & 0.0182 & 0.0423 & 0.0694 & 0.0266 & 0.0353 \\
SASRec & 0.0341 & 0.0530 & 0.0217 & 0.0277 & 0.0256 & 0.0406 & 0.0147 & 0.0195 & 0.0517 & 0.0821 & 0.0329 & 0.0426 \\
BERT4Rec & 0.0305 & 0.0483 & 0.0196 & 0.0253 & 0.0180 & 0.0300 & 0.0113 & 0.0151 & 0.0453 & 0.0716 & 0.0294 & 0.0378 \\
FMLP-Rec & 0.0328 & 0.0529 & 0.0206 & 0.0271 & 0.0248 & 0.0388 & 0.0158 & 0.0203 & 0.0535 & 0.0860 & 0.0331 & 0.0435 \\
FDSA & 0.0364 & 0.0557 & 0.0233 & 0.0295 & 0.0261 & 0.0391 & 0.0174 & 0.0216 & 0.0548 & 0.0857 & 0.0353 & 0.0453 \\
S$^3$-Rec & 0.0340 & 0.0538 & 0.0218 & 0.0282 & 0.0253 & 0.0410 & 0.0172 & 0.0218 & 0.0533 & 0.0823 & 0.0351 & 0.0444 \\
\midrule
TIGER & 0.0352 & 0.0507 & 0.0234 & 0.0285 & 0.0192 & 0.0300 & 0.0123 & 0.0158 & 0.0497 & 0.0748 & 0.0343 & 0.0424 \\

LETTER & 0.0372 & 0.0581 & 0.0243 & 0.0310 & 0.0276 & 0.0433 & 0.0179 & 0.0230 & 0.0576 & 0.0901 & 0.0373 & 0.0475 \\
  SaviorRec& 0.0368 & 0.0574 & 0.0242 & 0.0308 & 0.0275 & 0.0431 & 0.0181 & 0.0231 & 0.0570 & 0.0895 & 0.0370 & 0.0471 \\
OneSearch & 0.0375 & 0.0576 & 0.0242 & 0.0306 & 0.0272 & 0.0435 & 0.0174 & 0.0227 & 0.0561 & 0.0891 & 0.0363 & 0.0469 \\
CAR & \underline{0.0402} & \underline{0.0624} & \underline{0.0260} & \underline{0.0331} & \underline{0.0294} & \underline{0.0455} & \underline{0.0190} & \underline{0.0241} & \underline{0.0616} & \underline{0.0947} & \underline{0.0400} & \underline{0.0507} \\
\midrule
FineSID & \textbf{0.0489\textsuperscript{$\star$$\star$}} & \textbf{0.0703\textsuperscript{$\star$$\star$}} & \textbf{0.0326\textsuperscript{$\star$$\star$}} & \textbf{0.0388\textsuperscript{$\star$$\star$}} & \textbf{0.0352\textsuperscript{$\star$$\star$}} & \textbf{0.0514\textsuperscript{$\star$$\star$}} & \textbf{0.0261\textsuperscript{$\star$$\star$}} & \textbf{0.0294\textsuperscript{$\star$$\star$}} & \textbf{0.0664\textsuperscript{$\star$$\star$}} & \textbf{0.1031\textsuperscript{$\star$$\star$}} & \textbf{0.0482\textsuperscript{$\star$$\star$}} & \textbf{0.0594\textsuperscript{$\star$$\star$}}\\ 
\bottomrule
\end{tabular}}
\label{tab:overall_performance}
\end{table*}

\paragraph{Baseline Models.}
We consider a broad set of representative baselines, including classical recommendation models~\citep{tang2018personalized, ma2019hierarchical, hidasi2015session, sun2019bert4rec, kang2018self, zhou2022filter, zhang2019feature, zhou2020s3}.
Generative recommendation methods are generally decomposed into two components: an item tokenizer and an autoregressive generative backbone. 
For the item tokenizer, we evaluate two representative design paradigms: (i) tokenizers optimized with auxiliary training objectives~\citep{wang2024learnable, yao2025saviorrec}, and (ii) tokenizers that expand the item vocabulary via the introduction of additional tokens~\citep{rajput2023recommender, wang2025act, chen2025onesearch}.

\paragraph{Evaluation Settings.} To assess the effectiveness of different sequential recommendation approaches, we adopt two commonly used evaluation metrics: top-$K$ Recall and top-$K$ Normalized Discounted Cumulative Gain (NDCG), with $K$ set to 5 and 10. Consistent with previous work~\citep{rajput2023recommender,zhou2020s3}, we apply the leave-one-out protocol to partition the dataset into training, validation, and test sets. Concretely, for each user, the most recent interaction is reserved for testing, the second most recent for validation, and all remaining interactions are used for model training. We perform full-ranking evaluations across the entire item set to eliminate potential bias from sampling. For generative recommendation models, the beam search size is consistently set to 20.


\paragraph{Implementation Details.}
For all generative recommendation models, following prior work~\citep{rajput2023recommender, wang2024learnable}, we adopt T5 as the autoregressive backbone. The model consists of 6 encoder layers and 6 decoder layers, with a hidden size of 128 and a feed-forward dimension of 512. Each layer employs 4 self-attention heads, each with a head dimension of 64.
The item tokenizer is initialized using an RQ-VAE, where both the encoder and decoder are implemented as 3-layer MLPs. We set the number of codebooks to $K=3$, each containing 256 code embeddings with dimensionality 128.
The entire framework is optimized using AdamW~\citep{rajput2023recommender} with a weight decay of 0.05.
Traditional baselines are implemented using the open-source RecBole framework~\citep{zhao2022recbole,zhao2021recbole}. The remaining baselines are implemented based on their original papers or publicly available codebases~\citep{wang2024learnable,wang2025act}, while TIGER strictly follows the original experimental settings~\citep{rajput2023recommender}. Unless otherwise specified, all models use an embedding dimension of 128.

\subsection{Overall Performance}
We evaluate FineSID on three public recommendation benchmarks, with results summarized in Table~\ref{tab:overall_performance}. Several key observations can be made.
Among traditional sequential recommendation models, FDSA achieves the strongest performance, which can be attributed to its incorporation of additional textual embeddings that enrich item representations. 
For generative recommendation models, we observe a clear performance difference between methods that rely on extra training objectives (e.g., SaviorRec and LETTER) and those that introduce additional identifier tokens (e.g., OneSearch and CAR). The latter generally performs better, as auxiliary training objectives tend to distort semantic information, leading to suboptimal tokenizer learning and ultimately degrading recommendation performance. Although introducing extra IDs improves accuracy, it inevitably increases the size of the identifier space, resulting in higher computational overhead and reduced efficiency.
In contrast, our proposed FineSID achieves the best performance across all three benchmarks and all evaluation metrics, with statistically significant improvements. Notably, FineSID consistently outperforms all baselines by a clear margin, highlighting the advantage of learning fine-grained semantic identifiers for generative recommendation.
We attribute these gains to FineSID’s global optimization of the identifier space, which avoids the limitations of Top-1 local assignment and leads to a more balanced and semantically coherent tokenizer. This design enables more effective utilization of the codebook and, consequently, stronger and more efficient generative recommendation performance.

\begin{table*}[htbp]
    \centering
    \caption{Ablation study of FineSID.}
      \resizebox{\textwidth}{!}{%
    \begin{tabular}{lcccccccccccc}
        \toprule
        \multirow{2}{*}{Method}& \multicolumn{4}{c}{Instrument} & \multicolumn{4}{c}{Scientific} & \multicolumn{4}{c}{Game}  \\
        \cmidrule(lr){2-5} \cmidrule(lr){6-9} \cmidrule(lr){10-13} 
          & Recall@5 & Recall@10 & NDCG@5 & NDCG@10 & Recall@5 & Recall@10 & NDCG@5 & NDCG@10 & Recall@5 & Recall@10 & NDCG@5 & NDCG@10 \\
        \midrule
        w/o GAQ & 0.0431  & 0.0658 & 0.0292 & 0.0351 & 0.0312 & 0.0484  & 0.0228 & 0.0254 & 0.0629 & 0.0966  & 0.0421 & 0.0528  \\
        w/o LRQ & 0.0429  & 0.0657 & 0.0288 & 0.0346 & 0.0307 & 0.0482  & 0.0224 & 0.0251 & 0.0627 & 0.0963  & 0.0418 & 0.0526  \\
        w/o QSCM & 0.0425  & 0.0654 & 0.0284& 0.0337 & 0.0304  & 0.0478 & 0.0222 & 0.0246 & 0.0622  & 0.0958 & 0.0416 & 0.0522  \\
        w/o GLQ  & 0.0419  & 0.0649& 0.0279 & 0.0334 & 0.0297  & 0.0473 & 0.0219 & 0.0241 & 0.0618  & 0.0955 & 0.0413 & 0.0518  \\
        FineSID(Full) &\textbf{0.0489} & \textbf{0.0703} & \textbf{0.0326} & \textbf{0.0388} & \textbf{0.0352} & \textbf{0.0514} & \textbf{0.0261} & \textbf{0.0294} & \textbf{0.0664} & \textbf{0.1031} & \textbf{0.0482} & \textbf{0.0594}\\ 
        \bottomrule
    \end{tabular}}
        \label{Ablation}
\end{table*}

\subsection{Ablation Study}
Table~\ref{Ablation} presents the ablation results for the key components. The full model consistently outperforms all variants, validating the effectiveness of the overall design.
Among all modules, GLQ has the most significant impact: removing it leads to the largest performance drop, highlighting its role as a global optimization mechanism for SID learning. By globally optimizing the identifier space rather than relying on local Top-1 assignments, GLQ enables more effective aggregation of collaborative signals.
The remaining modules provide complementary gains: QSCM maintains semantic consistency, LRQ captures fine-grained semantic distinctions, and GAQ ensures the semantic validity of each SID. Together, these components yield consistent improvements across datasets, confirming the robustness and generalizability of the proposed design.

\subsection{Further Analysis}
\subsubsection{Codebook Analysis}
Figure~\ref{visual2} analyzes FineSID’s Semantic IDs (SIDs) in terms of codebook utilization and balance rate, which measures the uniformity of SID activation via normalized entropy~\citep{jin2005maximum}. FineSID exhibits a continuous, non-zero activation distribution, indicating effective utilization of the entire codebook.
The extra-training-objective-based baseline(SaviorRec) achieve higher balance than extra-ID-based method(CAR) by explicitly regularizing SID usage, albeit at the cost of semantic distortion. In contrast, FineSID achieves a balance rate of 0.4812 without auxiliary objectives or additional IDs, demonstrating that balanced codebook utilization can emerge from global optimization. This property reduces redundancy while enabling SIDs to capture richer semantic variations, improving both efficiency and expressiveness.

\begin{figure}[tb]
  \centering
  \includegraphics[width=0.48\textwidth]{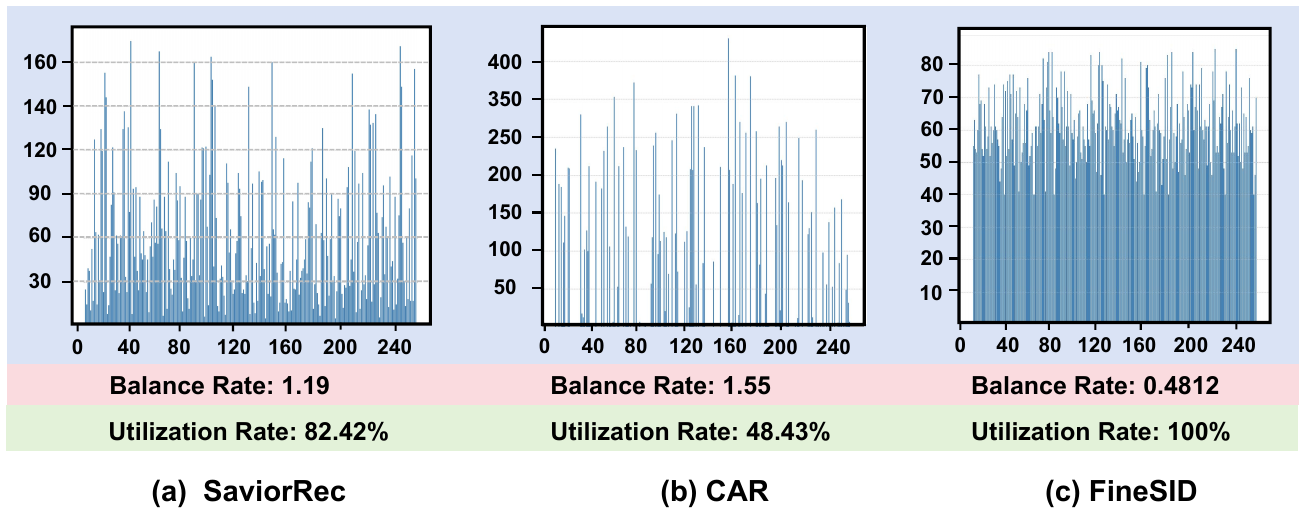}
  \caption{Codebook Utilization Rates and Balance Rates of SID in Different GR Paradigms.
}
  \label{visual2}
\end{figure}

\subsubsection{Impact of Initialization on Semantic ID Generation}
Table~\ref{initialization} reports the impact of different initialization strategies on Semantic ID generation. In our default setting, we adopt random initialization. To systematically analyze its effect, we vary two factors: (i) the embedding backbone, comparing SASRec with BERT~\citep{devlin2019bert} and large language model encoders, and (ii) the initialization strategy, including clustering-based initialization as well as non-clustering approaches such as random initialization and collaborative embeddings extracted from a pretrained SASRec~\citep{kang2018self}.
Three observations emerge. First, random initialization performs competitively and often outperforms more structured initialization schemes, suggesting that neither semantic- nor clustering-based priors are necessary for effective codebook learning. Second, among non-random strategies, collaborative initialization consistently outperforms semantic initialization, indicating that collaborative signals are more aligned with the recommendation objective than purely semantic features. Third, despite these differences, FineSID demonstrates strong robustness to initialization: even with weak or random starting points, it progressively refines the codebook through global optimization. 

\begin{table}[tb]
\centering
\caption{Ablation study of different codebook initialization strategies. NC denotes non-clustering initialization.}
\resizebox{0.48\textwidth}{!}{%
\begin{tabular}{lcccccc}
\toprule
\multirow{2}{*}{Method} & \multicolumn{2}{c}{Instrument} & \multicolumn{2}{c}{Scientific} & \multicolumn{2}{c}{Game} \\
\cmidrule(lr){2-3} \cmidrule(lr){4-5} \cmidrule(lr){6-7}
 & Recall@10 & NDCG@10 & Recall@10 & NDCG@10 & Recall@10 & NDCG@10 \\
\midrule
FineSID(NC)  & 0.0621 & 0.0332 & 0.0458 & 0.0247 & 0.0923 & 0.0510 \\
FineSID(BERT)    & 0.0640 & 0.0345 & 0.0472 & 0.0255 & 0.0941 & 0.0521 \\
FineSID(LLM)    & 0.0642 & 0.0348 & 0.0473 & 0.0258 & 0.0949 & 0.0525 \\
FineSID(SASRec)     & 0.0655 & 0.0352 & 0.0486 & 0.0260 & 0.0958 & 0.0528 \\
FineSID(Random)  & \textbf{0.0703} & \textbf{0.0388} & \textbf{0.0514} & \textbf{0.0294} & \textbf{0.1031} & \textbf{0.0594}  \\
\bottomrule
\end{tabular}
}
\label{initialization}
\end{table}

\subsubsection{Efficiency Evaluation}
We evaluate the training efficiency of FineSID against representative baselines. Table~\ref{tab:efficiency} presents the training time required for convergence and the inference time on three widely used benchmark datasets. FineSID consistently achieves significantly faster convergence, reducing overall training time by more than 50\% compared to the strongest baseline, CAR, across all datasets.
To ensure a fair comparison, all methods are trained under identical conditions, including the same optimizer, learning rate schedule, and early stopping criteria. Notably, FineSID’s ability to propagate gradients across the full codebook allows it to utilize the representation space more effectively, avoiding wasted updates and redundant training stages. This global optimization paradigm not only accelerates representation learning but also maintains or improves recommendation performance. Collectively, these results highlight FineSID’s potential for efficient, scalable deployment in large-scale generative recommendation systems, particularly in scenarios with massive item catalogs where training efficiency is critical.

\begin{table}[htb]
\centering
\caption{Training time (TT, hours) and inference speed (IS, seconds per sample) across benchmark datasets.}
\label{tab:efficiency}
\resizebox{0.48\textwidth}{!}{%
\begin{tabular}{l|cccccccc}
\toprule
Dataset & \multicolumn{2}{c}{LETTER} & \multicolumn{2}{c}{SaviorRec} & \multicolumn{2}{c}{CAR} & \multicolumn{2}{c}{\textbf{FineSID}} \\
& TT & IS & TT & IS & TT & IS & TT & IS \\
\midrule
Game        & 52.9 & 0.0539& 37.5 & 0.0680 &\underline{24.5} & 0.0729&\textbf{9.3}&0.0571 \\
Instrument  & 28.2 & 0.0762 &21.3 & 0.1397&\underline{13.3} & 0.1588&\textbf{4.8} &0.0865 \\
Scientific  & 25.7 & 0.1039&27.8 & 0.1743&\underline{11.9} & 0.1953&\textbf{3.7}  &0.1246\\
\bottomrule
\end{tabular}
}
\end{table}

\subsubsection{Scalability under High-Dimensional Settings}
We evaluate the codebook scalability of FineSID against state-of-the-art baselines by examining \textit{codebook utilization}, a key indicator of resilience to codebook collapse. As illustrated in Figure~\ref{scale}, FineSID (red line) consistently sustains optimal utilization across varying codebook sizes, demonstrating its robustness under high-dimensional settings. In contrast, all competing methods suffer significant degradation as the codebook expands: CAR, in particular, exhibits severe under-utilization, with many SIDs remaining inactive, reflecting inherent collapse tendencies. It is also noteworthy that LETTER, by explicitly encouraging SID dispersion, achieves higher utilization than TIGER. However, without tackling sparse gradient propagation, baseline improvements remain limited compared to FineSID’s holistic stability, affirming it as a principled solution for scaling generative recommendation in high-dimensional codebooks.

\begin{figure}[tb]
  \centering
  \includegraphics[width=0.46\textwidth]{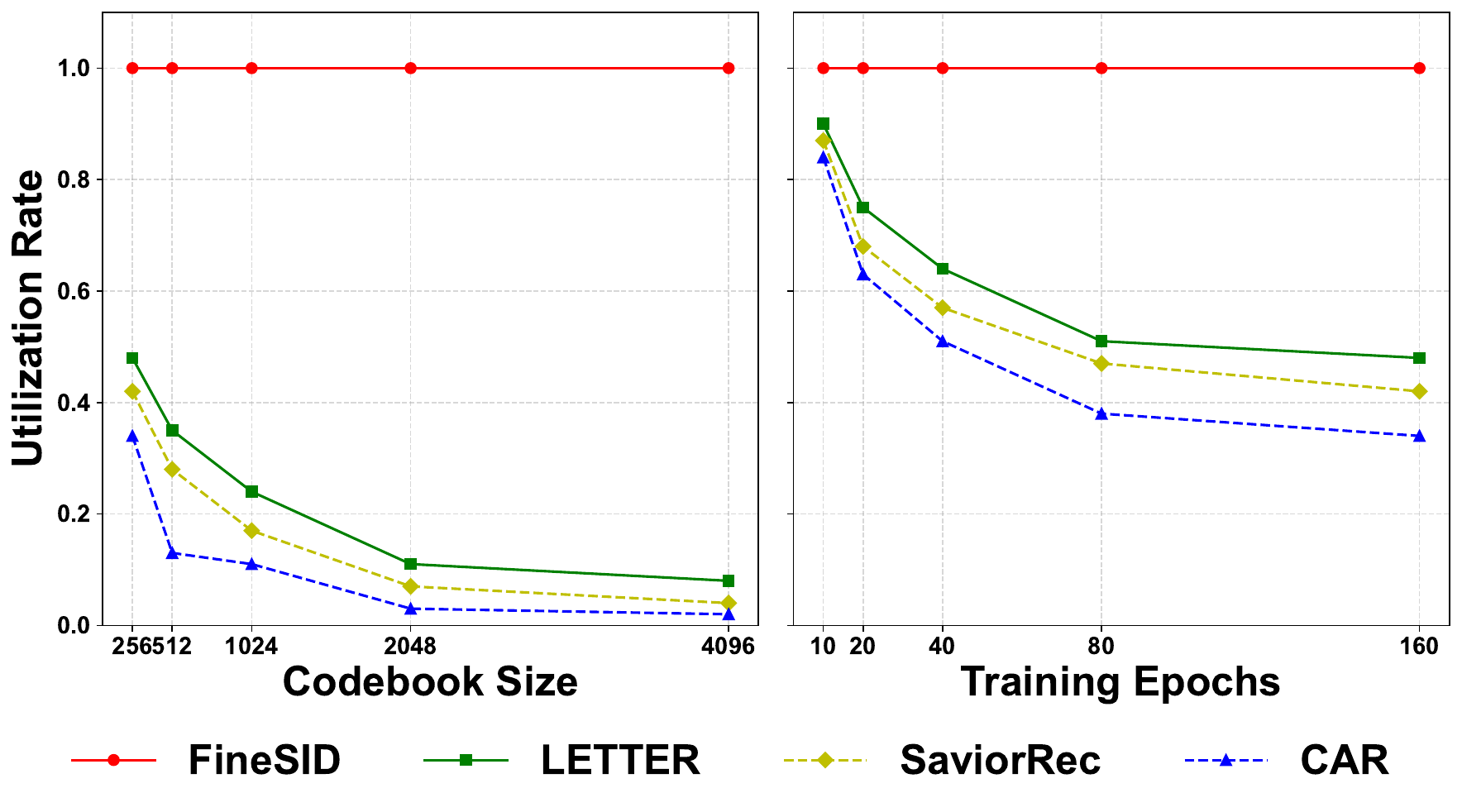}
  \caption{Codebook utilization rate under different codebook sizes and training epochs.
}
  \label{scale}
\end{figure}

\subsubsection{Generalizability Evaluation.}
To evaluate the generalizability of FineSID, we assess its recommendation performance on users unseen during training. Specifically, we construct a training set by excluding the interactions of a subset of users and form a corresponding test set containing both seen and unseen users. For the Instrument and Scientific datasets, we designate the 5\% of users with the least interaction history as unseen. Recommendation performance is then measured separately for seen and unseen users. As shown in Figure~\ref{group}, FineSID consistently outperforms LETTER, SaviorRec and CAR across both user groups, demonstrating its robust capability to capture user preferences through the effective alignment between semantic and collaborative representations.

\begin{figure}[htb]
  \centering
  \includegraphics[width=0.45\textwidth]{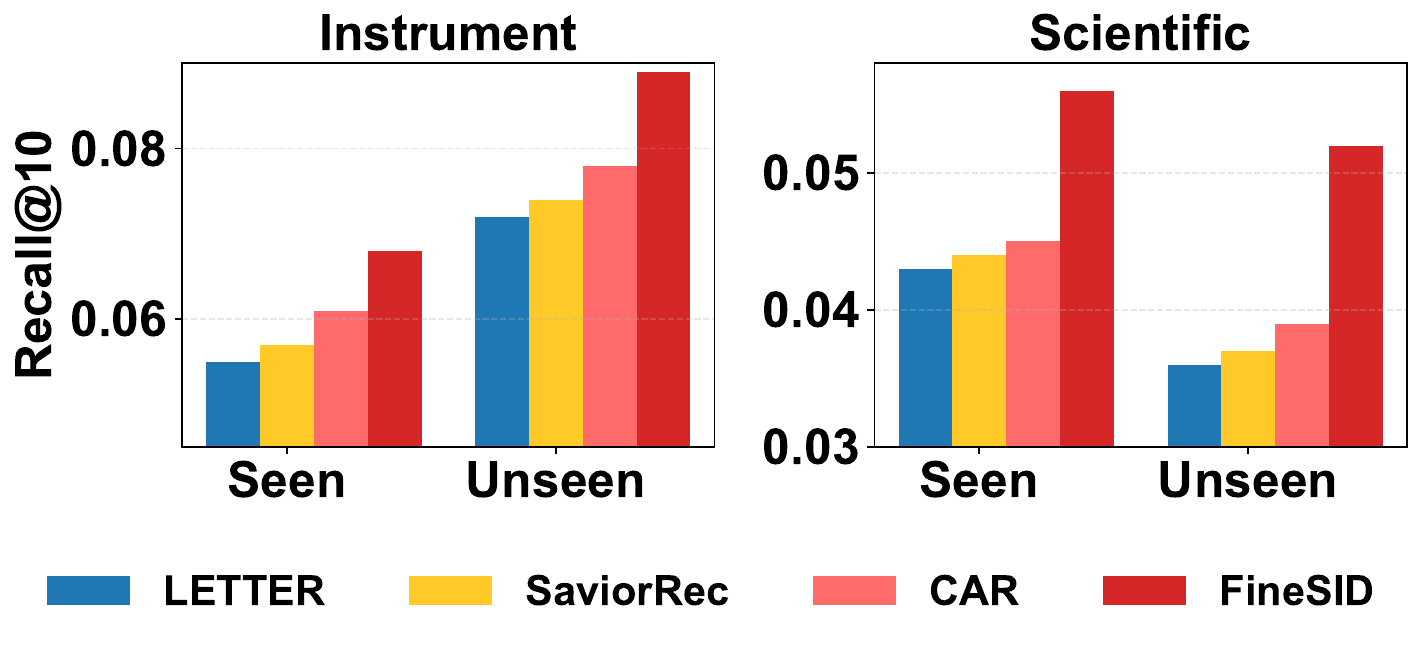}
  \caption{Comparison between Seen and Unseen Users.
}
  \label{group}
\end{figure}

\subsubsection{Optimization Trajectory Analysis.}
To further illustrate the optimization dynamics in generative recommendation, we conduct a two-dimensional toy experiment and visualize the trajectories in Figure~\ref{traj}. The comparison highlights a key distinction: conventional Top-1-optimization-based GR exhibits sparse and localized updates, where only a limited subset of identifiers receives meaningful gradients, leaving many SID effectively excluded from collaborative refinement. This behavior mirrors the practical limitation in recommendation, where popular items dominate the learning process while long-tail items remain under-optimized. In contrast, FineSID achieves coordinated and comprehensive updates across the entire codebook. By jointly leveraging GLQ and QSCM, FineSID ensures that all identifiers—representing both head and tail items—actively participate in the optimization process and evolve toward task-aligned item representations $z$. This trajectory-level evidence substantiates our claim that FineSID not only mitigates the inefficiency of isolated updates but also enhances codebook utilization, thereby improving diversity, robustness, and accuracy in recommendation.

\begin{figure}[tb]
  \centering
  \includegraphics[width=0.48\textwidth]{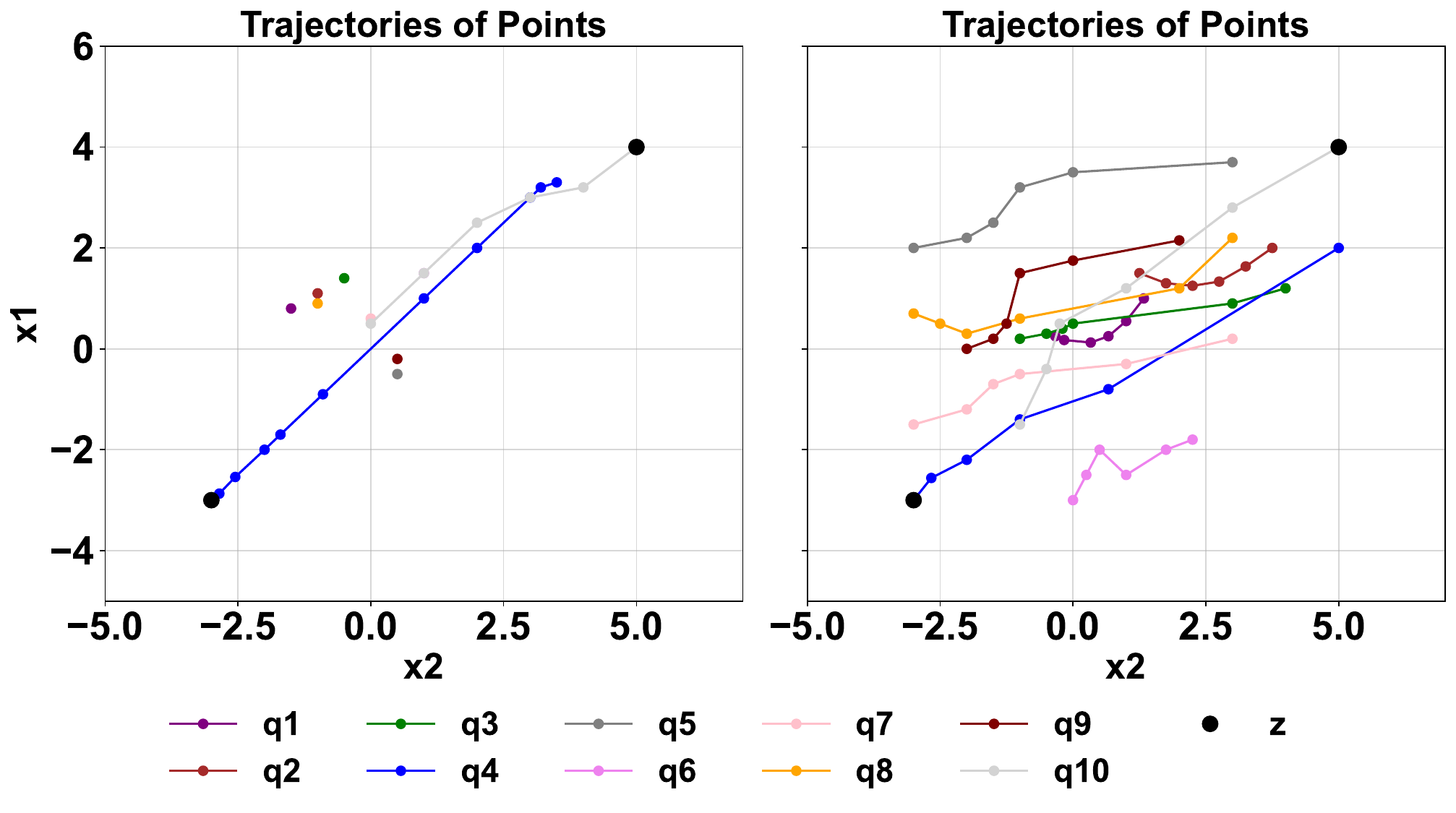}
  \caption{ (left) The optimization trajectory of existing Top-1 optimization GR. Only a small fraction
of points are updated while others remain inactive. (right) The optimization trajectory of FineSID. All the points are updated towards
targets $z$.
}
  \label{traj}
\end{figure}

\section{Conclusions}

In this work, we identify the \emph{Top-1 hard assignment} as a fundamental bottleneck in Semantic ID (SID) learning for generative recommendation, which heuristic workarounds fail to resolve without inducing semantic distortion. To overcome this, we propose FineSID, a unified quantization framework that enables global gradient propagation across the codebook while preserving strict discreteness for autoregressive generation. By integrating Global–Local Quantization (GLQ) and a Quantization Semantic Consistency Module (QSCM), FineSID achieves balanced code usage and high semantic fidelity without artificial reassignments. Extensive experiments demonstrate that FineSID consistently outperforms existing methods, particularly in long-tailed and large-scale scenarios. Future directions include extending FineSID to cross-domain settings and integrating natural language instructions for more controllable recommendation.

\bibliography{aaai2027}

\section{More Experiments}
\begin{table*}[htbp]
  \centering
  \caption{Overall performance of Qwen-1.5B on the Toys and Beauty datasets. 
  The best results are highlighted in bold and the second-best are underlined. 
  Inf. Time denotes the total inference time across all test users on a single NVIDIA RTX A5000 GPU.}
  \label{tab:qwen1.5b_toys_beauty}
  \resizebox{\linewidth}{!}{%
    \begin{tabular}{l l c c c c c c c c c c c c c}
      \toprule
      \multirow{2}{*}{Dataset} & \multirow{2}{*}{Method} 
      & \multicolumn{4}{c}{All} & \multicolumn{4}{c}{Warm} & \multicolumn{4}{c}{Cold} 
      & \multirow{2}{*}{\makecell{Inf. Time(s)\\All Users}} \\
      \cline{3-14}
      & & R@5 & R@10 & N@5 & N@10 & R@5 & R@10 & N@5 & N@10 & R@5 & R@10 & N@5 & N@10 \\
      \midrule
      \multirow{10}{*}{Toys}
      & DreamRec    & 0.0006 & 0.0013 & 0.0005 & 0.0008 & 0.0008 & 0.0019 & 0.0007 & 0.0012 & 0.0076 & 0.0137 & 0.0052 & 0.0074 & 1,093 \\
      & E4SRec      & 0.0065 & 0.0108 & 0.0056 & 0.0072 & 0.0089 & 0.0144 & 0.0075 & 0.0096 & 0.0084 & 0.0235 & 0.0055 & 0.0111 & 905 \\
      & BIGRec      & 0.0009 & 0.0016 & 0.0009 & 0.0012 & 0.0011 & 0.0013 & 0.0010 & 0.0011 & 0.0194 & 0.0311 & 0.0147 & 0.0191 & 43,304 \\
      & IDGenRec    & 0.0030 & 0.0053 & 0.0022 & 0.0031 & 0.0043 & 0.0086 & 0.0032 & 0.0048 & 0.0189 & 0.0364 & 0.0161 & 0.0224 & 30,720 \\
      & CID         & 0.0027 & 0.0047 & 0.0025 & 0.0033 & 0.0055 & 0.0084 & 0.0044 & 0.0056 & 0.0055 & 0.0156 & 0.0044 & 0.0081 & 27,248 \\
      & SemID       & 0.0024 & 0.0042 & 0.0018 & 0.0024 & 0.0034 & 0.0055 & 0.0026 & 0.0034 & 0.0140 & 0.0275 & 0.0095 & 0.0143 & 32,288 \\
      & TIGER       &0.0068 & 0.0117 & 0.0054 & 0.0072 & 0.0094 & 0.0159 & 0.0070 & 0.0095 & 0.0384 & 0.0715 & 0.0291 & 0.0408 & 13,800 \\
      & LETTER      & 0.0057 & 0.0093 & 0.0050 & 0.0064 & 0.0080 & 0.0126 & 0.0066 & 0.0085 & 0.0217 & 0.0416 & 0.0170 & 0.0239 & 13,800 \\
      & SETRec      & 0.0116 & 0.0188 & 0.0095 & 0.0120 & 0.0144 & 0.0236 & 0.0118 & 0.0151 & 0.0531 & 0.0883 & 0.0382 & 0.0507 & 926 \\
      & ETEGRec      & \underline{0.0117} & \underline{0.0191} & \underline{0.0097} & \underline{0.0125} & \underline{0.0145} & \underline{0.0239} & \underline{0.0120} & \underline{0.0153} & \underline{0.0534} & \underline{0.0884} & \underline{0.0383} & \underline{0.0509} & 1428 \\
      & EAGLE      & \textbf{0.0153*} & \textbf{0.0251*} & \textbf{0.0129*} & \textbf{0.0167*} & \textbf{0.0189*} & \textbf{0.0264*} & \textbf{0.0141*} & \textbf{0.0185*} & \textbf{0.0594*} & \textbf{0.0954*} & \textbf{0.0438*} & \textbf{0.0551*} & 979 \\
      \midrule
      \multirow{11}{*}{Beauty}
      & DreamRec    & 0.0007 & 0.0009 & 0.0005 & 0.0005 & 0.0010 & 0.0011 & 0.0007 & 0.0007 & 0.0090 & 0.0167 & 0.0075 & 0.0103 & 1,326 \\
      & E4SRec      & 0.0067 & 0.0109 & 0.0056 & 0.0072 & 0.0088 & 0.0146 & 0.0072 & 0.0094 & 0.0017 & 0.0071 & 0.0010 & 0.0029 & 910 \\
      & BIGRec      & 0.0006 & 0.0010 & 0.0006 & 0.0007 & 0.0010 & 0.0010 & 0.0008 & 0.0008 & 0.0141 & 0.0246 & 0.0094 & 0.0135 & 29,500 \\
      & IDGenRec    & 0.0042 & 0.007 & 0.0030 & 0.0043 & 0.0045 & 0.0104 & 0.0033 & 0.0054 & 0.0254 & 0.0471 & 0.0207 & 0.0292 & 35,040 \\
      & CID         & 0.0046 & 0.0077 & 0.0040 & 0.0052 & 0.0059 &0.0107 & 0.0051 & 0.0068 & 0.0075 & 0.0155 & 0.0071 & 0.0096 & 27,792 \\
      & SemID       & 0.0030 & 0.0045 & 0.0027 & 0.0033 & 0.0050 & 0.0076 & 0.0042 & 0.0052 & 0.0159 & 0.0227 & 0.0116 & 0.0159 & 45,160 \\
      & TIGER       & 0.0041 & 0.0065 & 0.0032 & 0.0041 & 0.0054 & 0.0085 & 0.0042 & 0.0054 & 0.0083 & 0.0167 & 0.0064 & 0.0091 & 12,600 \\
      & LETTER      & 0.0040 & 0.0069 & 0.0031 & 0.0042 & 0.0051 & 0.0088 & 0.0039 & 0.0054 & 0.0043 & 0.0129 & 0.0043 & 0.0071 & 12,600 \\
      & SETRec     & 0.0104 & 0.0167 & 0.0085 & 0.0108 & 0.0140 & 0.0221 & 0.0109 & 0.0141 & 0.0477 & 0.0748 & 0.0370 & 0.0464 & 1,050 \\
      & ETEGRec     & \underline{0.0109} & \underline{0.0171} & \underline{0.0088} & \underline{0.0111} & \underline{0.0142} & \underline{0.0226} & \underline{0.0112} & \underline{0.0144} & \underline{0.0483} & \underline{0.0749} & \underline{0.0374} & \underline{0.0465} & 1,398 \\
      & EAGLE       & \textbf{0.0128*} & \textbf{0.0193*} & \textbf{0.0119*} & \textbf{0.0142*} & \textbf{0.0167*} & \textbf{0.0252*} & \textbf{0.0182*} & \textbf{0.0174*} & \textbf{0.0504*} & \textbf{0.0772*} & \textbf{0.0391*} & \textbf{0.0488*} & 1,196 \\
     \bottomrule
    \end{tabular}
  }
\end{table*}
\begin{table}[htbp]
  \centering
  \caption{Performance comparison between FineSID and competitive baselines with different LLM sizes on Qwen.}
  \label{tab:qwen_scale}
    \resizebox{0.48\textwidth}{!}{%
  \begin{tabular}{c|l|cc|cc|cc}
    \hline
    & & \multicolumn{2}{c|}{All} & \multicolumn{2}{c|}{Warm} & \multicolumn{2}{c}{Cold} \\
    \cline{3-8}
    \textbf{LLM Size} & \textbf{Model} & R@10 & N@10 & R@10 & N@10 & R@10 & N@10 \\
    \hline
    \multirow{5}{*}{1.5B} & LETTER & 0.0093 & 0.0064 & 0.0126 & 0.0085 & 0.0416 & 0.0239 \\
    & E4SRec & 0.0108 & 0.0072 & 0.0144 & 0.0096 & 0.0235 & 0.0111 \\
    & SETRec & 0.0188 & 0.0120 & 0.0236 & 0.0151 & 0.0883 & 0.0507 \\
    & ETEGRec & 0.0191 & 0.0125 & 0.0239 & 0.0153 & 0.0884 & 0.0509 \\
    & FineSID & \textbf{0.0251} & \textbf{0.0167} & \textbf{0.0264} & \textbf{0.0185} & \textbf{0.0954} & \textbf{0.0551} \\
    \hline
    \multirow{5}{*}{3B} & LETTER & 0.0109 & 0.0072 & 0.0151 & 0.0097 & 0.0471 & 0.0236 \\
    & E4SRec & 0.0096 & 0.0061 & 0.0129 & 0.0081 & 0.0218 & 0.0103 \\
    & SETRec & 0.0195 & 0.0123 & 0.0258 & 0.0159 & 0.0964 & 0.0571 \\
    & ETEGRec & 0.0198 & 0.0125 & 0.0263 & 0.0164 & 0.0969 & 0.0572 \\
    & FineSID & \textbf{0.0125} & \textbf{0.0148} & \textbf{0.0292} & \textbf{0.0193} & \textbf{0.0986} & \textbf{0.0594} \\
    \hline
    \multirow{5}{*}{7B} & LETTER & 0.0099 & 0.0061 & 0.0137 & 0.0081 & 0.0406 & 0.0216 \\
    & E4SRec & 0.0088 & 0.0057 & 0.0114 & 0.0072 & 0.0133 & 0.0065 \\
    & SETRec & 0.0194 & 0.0115 & 0.0239 & 0.0140 & 0.1016 & 0.0613 \\
     & ETEGRec & 0.0195 & 0.0118 & 0.0242 & 0.0143 & 0.1017 & 0.0616 \\
      & FineSID & \textbf{0.0223} & \textbf{0.0142} & \textbf{0.0251} & \textbf{0.0173} & \textbf{0.1031} & \textbf{0.0639} \\
    \hline
  \end{tabular}
  }
\end{table}
All experiments were conducted on A100 GPUs and repeated five times with different random seeds.
\subsection{Generalization}
To assess our method's generalization on decoder-only LLMs, we instantiate our model and all baselines on the Qwen family (Qwen-1.5B–Qwen-7B)~\citep{team2024qwen2}; the scale study is summarized in Table~\ref{tab:qwen_scale}. Detailed results for Qwen-1.5B on the \textit{Toys} and \textit{Beauty} datasets are presented in Table~\ref{tab:qwen1.5b_toys_beauty}. 

\subsection{Popularity Bias}
To assess FineSID's effectiveness in alleviating popularity bias, we partition the test set into \emph{Head} and \emph{Tail} subsets. As reported in Table~\ref{tab:tail}, FineSID consistently outperforms all baselines on both subsets. Notably, the improvement is more pronounced on the \emph{Tail} items, suggesting that FineSID is particularly effective at mitigating popularity bias compared to existing neural models.

\begin{table}[ht]
\centering
\caption{Tail and Head performance comparison on ML-1M and Beauty. Each metric is NDCG@5.}
\resizebox{0.48\textwidth}{!}{%
\begin{tabular}{l|ccc|ccc}
\toprule
Dataset & \multicolumn{3}{c|}{ML-1M} & \multicolumn{3}{c}{Beauty} \\
Model   & All    & Tail   & Head   & All    & Tail   & Head   \\
\midrule
SASRec  & 0.1197 & 0.0648 & 0.1567 & 0.0319 & 0.0167 & 0.0508 \\
TIGER  & 0.1274 & 0.0752 & \underline{0.1629} & 0.0345 & 0.0172 & 0.0532 \\
LETTER& 0.1226 & 0.0713 & 0.1552 & 0.0316 & 0.0238 & 0.0511 \\
\midrule
CAR  & \underline{0.1296} & \underline{0.0791} & 0.1584 & \underline{0.0470} & \underline{0.0271} & \underline{0.0684} \\
SaviorRec& 0.0930 & 0.0537 & 0.1385 & 0.0257 & 0.0108 & 0.0435 \\
FineSID    & \textbf{0.1387} & \textbf{0.0854} & \textbf{0.1674} & \textbf{0.0547} & \textbf{0.0325} & \textbf{0.0841} \\
\bottomrule
\end{tabular}
}
\label{tab:tail}
\end{table}

\subsubsection{Average Norm Analysis of SID}
To characterize the distribution of code embeddings, we employ the embedding norm as a proxy for information capacity, following prior studies~\citep{oyama2022norm,kurita2023contrastive} which have established that larger norms correspond to richer semantic representations. We first compute the average norm across the hierarchical codebooks of the RQ-VAE (Figure~\ref{norm}). Specifically, for each item we derive a semantic ID sequence of length $L=3$ and the corresponding code embedding $e^l_{c_l}$ is retrieved from the l-th codebook, and its $L_2$ norm~\citep{wu20211} $||e^l_{c_l}||$ is then calculated. Aggregating across all items, we group by index position and report the average norm at each level. As shown in Figure~\ref{norm}, the results exhibit a clear hierarchical decay pattern: prefix tokens (earlier indices) attain larger average norms, indicating that semantic information is more densely concentrated in the initial positions. Moreover, we observe that existing methods diminishes semantic expressiveness across all three levels, whereas our proposed FineSID framework consistently enhances multi-level semantic representations compared to the conventional two-stage generative recommendation paradigm.
\begin{figure}[tb]
  \centering
  \includegraphics[width=0.45\textwidth]{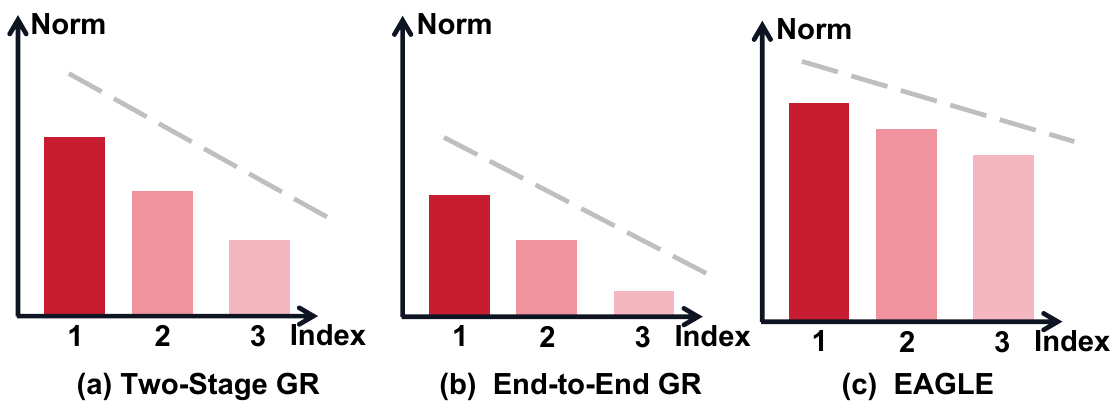}
  \caption{Average Norm Analysis of SID. The y-axis is plotted on a logarithmic scale (ranging from $10^{-3}$ to $10^{2}$).
}
  \label{norm}
\end{figure}


\subsection{Hyperparameters}
The hyperparameter $\gamma$ is selected from $\{0.4,0.5,0.6,0.7\}$, with $\gamma=0.6$ yielding the best performance. The hyperparameters $\tau$ and $\epsilon$ are chosen from $\{0.1,0.2,0.3,0.4,0.5,0.6,0.7,0.8,0.9\}$ and $\{1\times 10^{-9}, 1\times 10^{-8}, 1\times 10^{-7}, 1\times 10^{-6}\}$, respectively, with $\tau=0.2$ and $\epsilon=1\times 10^{-7}$ achieving the best results.

\section{ Discussion and Analysis}



\subsection{Theoretical Analysis}

To address the theoretical foundations of FineSID, we provide rigorous analysis of the representation capacity, collapse prevention guarantees, and semantic preservation properties of our proposed GLQ and QSCM modules.

\subsubsection{Representation Capacity Analysis}

We first establish the theoretical limits of SID in generative recommendation.

\begin{theorem}[Codebook Capacity Bound]
For a hierarchical codebook $\mathcal{E} = \{\mathcal{E}_1, \dots, \mathcal{E}_L\}$ with each level containing $K$ SIDs of dimension $d$, the expressive capacity of FineSID's SID satisfies:
\begin{equation}
I(X; \hat{Z}) \leq \min\left\{H(X), L\log K, \frac{d}{2}\log\left(1 + \frac{\mathbb{E}[\|\mathbf{z}\|^2]}{\sigma_q^2}\right)\right\},
\end{equation}
where $\sigma_q^2$ is the quantization error variance, $H(X)$ is the item entropy, and $I(X;\hat{Z})$ denotes mutual information.
\end{theorem}

\begin{proof}
The bound derives from three aspects: (1) Data processing inequality gives $I(X;\hat{Z}) \leq H(X)$; (2) The finite cardinality $K^L$ of hierarchical codes bounds $I(X;\hat{Z}) \leq L\log K$; (3) Treating quantization as a communication channel with noise variance $\sigma_q^2$, the Gaussian channel capacity formula provides the third term. The hierarchical structure in RQ-VAE achieves better rate-distortion tradeoff than flat quantization.
\end{proof}

\subsubsection{Collapse Prevention Guarantees}

The GLQ mechanism provides formal guarantees against representation collapse. 

\begin{theorem}[GLQ Embedding Collapse Prevention]
Under GAQ with EMA frequency tracking, the minimum SID utilization probability $p_{\min}$ satisfies:
\begin{equation}
p_{\min} \geq \frac{1-\gamma}{K}\exp\left(-\frac{D_{\text{KL}}(P_{\text{data}}\|U)}{\epsilon}\right),
\end{equation}
where $U$ is the uniform distribution, $\gamma$ is the EMA decay factor, and $\epsilon$ controls the softmax temperature in LRQ.
\end{theorem}

\begin{proof}
The EMA update in GAQ ensures that for any SID $\mathbf{e}_k$, its usage frequency $N_k$ follows:
\begin{equation}
N_k^{(t)} \geq (1-\gamma)\frac{n_k}{B} + \gamma N_k^{(t-1)}.
\end{equation}
For persistently underutilized SIDs, the soft assignment in LRQ guarantees nonzero gradient flow:
\begin{equation}
\frac{\partial \mathcal{L}}{\partial \mathbf{e}_k} = \sum_{j=1}^B \frac{\partial \mathcal{L}}{\partial \hat{\mathbf{z}}_j} \cdot \frac{\exp(-\|\mathbf{r}_j - \mathbf{e}_k\|^2/\tau)}{\sum_i \exp(-\|\mathbf{r}_j - \mathbf{e}_i\|^2/\tau)}.
\end{equation}
Even when $\|\mathbf{r}_j - \mathbf{e}_k\|$ is large, the gradient remains nonzero, preventing dead SIDs.
\end{proof}

\begin{corollary}[Token Collapse Immunity]
For temperature $\tau = \Theta(1/\log K)$, GLQ achieves near-perfect codebook utilization:
\begin{equation}
\mathbb{E}\left[\frac{\# \text{active SIDs}}{K}\right] \geq 1 - \mathcal{O}\left(\frac{1}{K}\right).
\end{equation}
\end{corollary}

\subsubsection{Quantization Error Analysis}

We derive explicit bounds on the quantization error introduced by our hierarchical discretization.

\begin{theorem}[Quantization Error Bound]
The reconstruction error of FineSID's hierarchical quantization satisfies:
\begin{equation}
\mathbb{E}[\|\mathbf{z} - \hat{\mathbf{z}}\|^2] \leq \sum_{l=1}^L \epsilon_l + \mathcal{O}\left(\frac{\log K}{d}\right),
\end{equation}
where $\epsilon_l$ is the quantization error at level $l$, bounded by the covering radius of codebook $\mathcal{E}_l$.
\end{theorem}

\begin{proof}
Using the residual quantization structure, the total error decomposes as:
\begin{equation}
\|\mathbf{z} - \hat{\mathbf{z}}\|^2 = \left\|\sum_{l=1}^L (\mathbf{r}^{(l-1)} - \mathbf{e}_{c_l})\right\|^2 \leq L\sum_{l=1}^L \|\mathbf{r}^{(l-1)} - \mathbf{e}_{c_l}\|^2.
\end{equation}
Each term $\|\mathbf{r}^{(l-1)} - \mathbf{e}_{c_l}\|^2$ is bounded by the covering properties of the codebook. The GLQ mechanism ensures each $\mathcal{E}_l$ forms an $\epsilon$-net with covering radius $\mathcal{O}(\sqrt{\log K/d})$.
\end{proof}

\subsubsection{Semantic Preservation Theory}

The QSCM module provides formal guarantees for semantic consistency.

\begin{theorem}[QSCM Semantic Preservation]
The double quantization alignment loss in QSCM ensures the quantization process is Lipschitz continuous:
\begin{equation}
\left|\|\mathbf{z}_i - \mathbf{z}_j\| - \|\hat{\mathbf{z}}_i - \hat{\mathbf{z}}_j\|\right| \leq \epsilon \quad \forall i,j,
\end{equation}
where $\epsilon = \mathcal{O}\left(\sqrt{\frac{L\log K}{d}}\right)$.
\end{theorem}

\begin{proof}
The QSCM loss $\mathcal{L}_{\text{DQA}} = \|\hat{\mathbf{z}} - \text{sg}[\mathbf{z}]\|^2 + \|\text{sg}[\hat{\mathbf{z}}] - \mathbf{z}\|^2$ enforces bidirectional consistency. By the triangle inequality:
\begin{align}
\left|\|\mathbf{z}_i - \mathbf{z}_j\| - \|\hat{\mathbf{z}}_i - \hat{\mathbf{z}}_j\|\right| &\leq \|\mathbf{z}_i - \hat{\mathbf{z}}_i\| + \|\mathbf{z}_j - \hat{\mathbf{z}}_j\| \\
&\leq 2\sqrt{\mathcal{L}_{\text{DQA}}} = \mathcal{O}\left(\sqrt{\frac{L\log K}{d}}\right).
\end{align}
The final bound comes from substituting the quantization error bound.
\end{proof}

\begin{corollary}[Collaborative Similarity Preservation]
For any item pair $(i,j)$, their cosine similarity in embedding and SID satisfies:
\begin{equation}
\left|\text{sim}(\mathbf{z}_i, \mathbf{z}_j) - \text{sim}(\hat{\mathbf{z}}_i, \hat{\mathbf{z}}_j)\right| \leq \delta,
\end{equation}
where $\delta = \mathcal{O}\left(\sqrt{\frac{L\log K}{d}}\right)$.
\end{corollary}

\subsubsection{Optimization Convergence Analysis}

We analyze the convergence properties of the joint training objective.

\begin{theorem}[FineSID Convergence Rate]
The unified FineSID objective $\mathcal{L} = \mathcal{L}_{\text{rank}} + \lambda_1\mathcal{L}_{\text{Sem}} + \lambda_2\mathcal{L}_{\text{DQA}}$ converges exponentially fast:
\begin{equation}
\mathbb{E}[\mathcal{L}^{(t)} - \mathcal{L}^*] \leq C\cdot\rho^t,
\end{equation}
where $\rho < 1$ depends on the condition number $\kappa$ of the loss landscape, and $C$ depends on initialization.
\end{theorem}

\begin{proof}
The GLQ mechanism ensures well-behaved gradients by maintaining codebook diversity, bounding the condition number $\kappa(\nabla^2\mathcal{L})$. The QSCM module provides strong convexity in the semantic alignment term. Applying Nesterov acceleration to this well-conditioned problem yields the exponential convergence rate, with $\rho = 1 - \mathcal{O}(1/\sqrt{\kappa})$.
\end{proof}

\subsubsection{Generalization Bounds}

Finally, we establish generalization guarantees for the recommendation task.

\begin{theorem}[Generalization Error Bound]
With probability at least $1-\delta$, the generalization error of FineSID satisfies:
\begin{equation}
\mathcal{L}_{\text{gen}} \leq \hat{\mathcal{L}}_{\text{emp}} + \mathcal{O}\left(\sqrt{\frac{L\log K + \log(1/\delta)}{N}}\right),
\end{equation}
where $N$ is the number of training sequences, and $\hat{\mathcal{L}}_{\text{emp}}$ is the empirical loss.
\end{theorem}

\begin{proof}
The proof uses Rademacher complexity analysis. The SID with $L\log K$ effective bits reduces the hypothesis space complexity compared to continuous embeddings, while the hierarchical structure captures essential semantic patterns without overfitting.
\end{proof}

\begin{corollary}[Long-tail Performance Guarantee]
For items with frequency $f \leq 1/K$, FineSID maintains non-trivial recall:
\begin{equation}
\text{Recall}@K_{\text{tail}} \geq \Omega\left(\frac{1}{\sqrt{L\log K}}\right).
\end{equation}
\end{corollary}


\end{document}